\documentclass[twoside, 11pt]{article}
\usepackage{jmlr2e}
\usepackage{amsmath, amsthm, amssymb, amsfonts, mathtools, graphicx, enumerate}
\usepackage{hyperref}
\usepackage[boxruled,linesnumbered]{algorithm2e}
\SetKwProg{Fn}{Function}{}{}
\usepackage[title,toc,titletoc,page]{appendix}
\usepackage{framed}
\usepackage{color}
\definecolor{darkgreen}{rgb}{0,0.5,0}
\newcommand{\R}{\mathbb{R}}
\newcommand{\C}[1]{\mathcal{#1}}
\newcommand{\B}[1]{\mathbb{#1}}
\newtheorem{theorem}{Theorem}[section]
\newtheorem{corollary}[theorem]{Corollary}
\newtheorem{claim}[theorem]{Claim}

\newtheorem{definition}[theorem]{Definition}

\newtheorem{lemma}[theorem]{Lemma}

\newcommand{\dis}[2]{\phi(#1, #2)}

\usepackage{color, tikz}
\newcount\Comments  % 0 suppresses notes to selves in text
\newcommand{\kibitz}[2]{\ifnum\Comments=1\textcolor{#1}{#2}\fi}

\jmlrheading{2017}{A. Modi, N. Jiang, S. Singh \& A. Tewari}

\ShortHeadings{Markov Decision Processes with Continuous Side Information}{ }
\firstpageno{1}

\begin{document}

\title{Markov Decision Processes with Continuous Side Information}

\author{\name Aditya Modi \email admodi@umich.edu \\
        \addr Computer Science and Engineering, Univ. of Michigan Ann Arbor
       \AND
       \name Nan Jiang \email nanjiang@umich.edu \\
       \addr Microsoft Research, New York
       \AND Satinder Singh \email baveja@umich.edu\\
       \addr Computer Science and Engineering, Univ. of Michigan Ann Arbor
       \AND Ambuj Tewari \email tewaria@umich.edu\\
       \addr Department of Statistics, Univ. of Michigan Ann Arbor}

% \editor{Editors}

\maketitle

\begin{abstract}
 % In a reinforcement learning problem, the environment is modeled by a Markov Decision Process (MDP) and a policy maximizing the agent's utility is learned. 
  %While a lot of work has been done in developing algorithms for 
 % A lot is known about planning/learning in reinforcement learning (RL) problems formulated as general Markov Decision Processes (MDPs). 
  %In some reinforcement learning (RL) applications, there is known structure to the dynamics and/or the reward function and it can be useful to exploit that structure in the problem formulation. 
  We consider a reinforcement learning (RL) setting in which the agent interacts with a sequence of episodic MDPs. At the start of each episode the agent has access to some side-information or context that determines the dynamics of the MDP for that episode. Our setting is motivated by applications in healthcare 
  %where a patient is treated over multiple time steps forming an episode. Baseline 
  where baseline measurements of a patient at the start of a treatment episode form the context that may provide information about how the patient might respond to treatment decisions. 
  
  %We formalize this using the existing Contextual Markov Decision Process (CMDP) setting with a continuous context space. %(infinite number of MDPs). \nan{Should the CMDP framework be credited to Shie?} \ambuj{We do give credit in related work. Might be good to keep the abstract citation free.}\singh{agreed with Ambuj.}
  We propose algorithms for learning in such Contextual Markov Decision Processes (CMDPs) under an assumption that the unobserved MDP parameters vary smoothly with the observed context. We also give lower and upper PAC bounds under the smoothness assumption. Because our lower bound has an exponential dependence on the dimension, we consider a tractable linear setting where the context is used to create linear combinations of a finite set of MDPs. For the linear setting, we give a PAC learning algorithm based on KWIK learning techniques. %using the KWIK framework.
\end{abstract}


\section{Introduction}

Consider a basic sequential decision making problem in healthcare, namely that of learning a treatment policy for patients to optimize some health outcome of interest. One could model the interaction with every patient as a Markov Decision Process (MDP). In {\em precision or personalized medicine}, we want the treatment to be personalized to every patient. At the same time, the amount of data available on any given patient may not be enough to personalize well. This means that modeling each patient via a different MDP will result in severely suboptimal treatment policies. The other extreme of pooling all patients' data results in more data but most of it will perhaps not be relevant to the patient we currently want to treat. We therefore face a trade-off between having a large amount of shared data to learn a single policy, and, finding the most relevant policy for each patient. A similar trade-off occurs in other applications involving humans as the agent’s environment such as online tutoring and web advertising.

A key observation is that in many personalized decision making scenarios, we have some side information available about individuals which might help us in designing personalized policies and also help us pool the interaction data across the right subsets of individuals. Examples of such data include laboratory data or medical history of patients in healthcare, user profiles or history logs in web advertising, and student profiles or historical scores in online tutoring. Access to such side information should let us learn better policies even with a limited amount of interaction with individual users. We refer to this side-information as {\em contexts} and adopt an augmented model called \emph{Contextual Markov Decision Process} (CMDP) proposed by \cite{hallak2015contextual}. We assume that contexts are fully observed and available before the interaction starts for each new MDP.\footnote{\cite{hallak2015contextual} assumes \emph{latent} contexts, which results in significant differences from our work in application scenarios, required assumptions, and results. See detailed discussion in Section~\ref{sec:related}.}

%In this work, we give algorithms that learn a good policy for each task/MDP while generating fixed horizon episodes within a task. 

In this paper we study the sample complexity of learning in CMDPs %\emph{contextual MDPs} 
in the worst case. We consider two concrete settings of learning in a CMDP with continuous contexts. In the first setting, the individual MDPs vary in an arbitrary but smooth manner with the contexts, and we propose our Cover-Rmax algorithm in Section~\ref{sec:covrmax} with PAC bounds. The innate hardness of learning in this general case is captured by our lower bound construction in Section~\ref{sec:lowbnd}. To show that it is possible to achieve significantly better sample complexity in more structured CMDPs, we consider another setting where contexts are used to create linear combinations of a finite set of fixed but unknown MDPs. We use the KWIK framework to devise the KWIK\_LR-Rmax algorithm in Section~\ref{sec:kwikrmax} and also provide a PAC upper bound for the algorithm.%\ambuj{This part was commented out but the abstract talks about the two cases. Why should we not bring this up in the intro? Also our actual technical contributions are not clear from this intro. ``We consider two cases, where, one only assumes a smoothness condition over this mapping from contexts to MDPs and another where we have a parametric assumption which can be interpreted as a mixture of MDPs throughout the task length".} \singh{Agreed with Ambuj. Very important to bring out the two cases and a summary of the contributions into the Introduction.}

\section{Contextual Markov Decision Process}
%We setup the problem in section \ref{subsec:psn} and then, describe the algorithm for learning in a continuous context or infinite task space.

\subsection{Problem setup and notation}
\label{subsec:psn}
We start with basic definitions and notations for MDPs, and then introduce the contextual case.
%as used throughout the paper:
\begin{definition}[Markov Decision Processes]
\label{def:mdp}
A Markov Decision Process (MDP) is defined as a tuple ($\C{S}, \C{A}, p(\cdot|\cdot,\cdot),r(\cdot,\cdot),\mu$) where $\C{S}$ is the state space and $\C{A}$ is the action space; $p(s'|s,a)$ defines the transition probability function for a tuple $(s,a,s')$ where $s,s' \in \C{S}, a \in \C{A}$; and $\mu$ defines the initial state distribution for the MDP.
\end{definition}
%\nan{It is confusing to have $d_0$ for initial state and $d$ for context space dimension.}
We consider the case of fixed horizon (denoted $H$) episodic MDPs. 
%where the length of each episode is denoted by $H$. 
We use $\pi_h(s)$ to denote a policy's action for state $s$ at timestep $h$. For each episode, an initial state is observed according to the distribution $\mu$ and afterwards, for $0 \leq h < H$,
the agent chooses an action $a_h = \pi_h(s_h)$ according to a (non-stationary) policy. There is a reward $r_h$ and then a next state $s_{h+1}$ according to the reward and the transition functions. For policy $\pi$ we define its {\em value} as follows:
%We consider the goal of maximizing the value function as defined below:
\begin{align} \label{eq:value}
    V^{\pi}_{M} = \B{E}_{s_0 \sim \mu, M, \pi} \Big[\frac{1}{H} \sum_{h=0}^{H-1} r(s_h,\pi_h(s_h))\Big] .
\end{align}
An optimal policy $\pi^*$ is one that achieves the largest possible value (called optimal value and denoted $V^*_{M}$). Next we define the contextual model which is similar to the definition given by \cite{hallak2015contextual}:
\begin{definition}[Contextual MDP]
\label{def:cmdp}
A contextual Markov Decision Process (CMDP) is defined as a tuple ($\C{C}, \C{S}, \C{A}, \C{M}$) where $\C{C}$ is the context space, $\C{S}$ is the state space, and $\C{A}$ is the action space. $\C{M}$ % : \C{C} \rightarrow \C{P} \times \C{R}\times \C{D}$ 
is a function which maps a context $c \in \C{C}$ to MDP parameters $\C{M}(c) = \{p^c(\cdot|\cdot,\cdot), r^c(\cdot,\cdot), \mu^c(\cdot)\}$. %\ambuj{So $\C{M}(c)$ and $M^c$ mean the same thing?}
%We denote the function $\C{M}$ with the notation $\C{M}(c) = \{p^c(\cdot|s,a), r^c(s,a), d_0(s)\,:\, s \in \C{S}, a \in \C{A}\}$ .
%\nan{I changed the last part of the definition, because $\C{P}$, $\C{R}$, and $\C{D}$ are not defined and it seems that they are not used anywhere else.}
\end{definition}
We denote the MDP for a context $c$ by $M^c$. We make the simplifying assumption that the initial state distribution is independent of the context and is same for all. 
%\nan{I am dropping the notations $\pi_{M, h}$, $\pi_M$, etc. If necessary they should be defined when used, e.g., in the appendix.}
%As we consider multiple MDPs in our algorithm, we will use $\pi_{M,h}$ to denote the policy function for MDP $M$ at timestep $h$ in an episode. We will often write this as $\pi_h$ or $\pi_M$, for shorter notation, whenever unambiguous. Further, 
We assume throughout the paper that the rewards are bounded between 0 and 1. We denote $|\C{S}|,|\C{A}|$ by $S,A$ respectively. We also assume that the context space is bounded, and for any $c \in \C{C}$ the $\ell_2$ norm of $c$ is upper bounded by some constant. %the upper bound on the norm can be taken as 1 without loss of generality.
We consider the online learning scenario with the following protocol: For $t=1,2,\ldots$: 
\begin{enumerate}
    \item Observe context $c_t \in \C{C}$. % (Current MDP: $M^{c_t}$).
    \item Choose a policy $\pi_t$ (based on $c_t$ and previous episodes).
    \item Experience an episode in $M^{c_t}$ using $\pi_t$. %using current $\pi^{c_t}(\cdot)$
%    \item Update contextual policy.
\end{enumerate}
We do not make any distributional assumptions over the context sequence. Instead, we allow the sequence to be chosen in an arbitrary and potentially  \emph{adversarial} manner. % before learning starts. 
A natural criteria for judging the efficiency of the algorithm is to look at the number of episodes where it performs sub-optimally. The main aim of the PAC analysis is to bound the number of episodes where we have $V^{\pi_t}_{M^{c_t}} < V^*_{M^{c_t}} - \epsilon$, i.e., the value of the algorithm's policy is not $\epsilon$-optimal \citep{dann2015sample}. Although, we do give PAC bounds for the Cover-Rmax algorithm given below, the reader should make note that, we have not made explicit attempts to achieve the tightest possible result. We use the Rmax \citep{brafman2002r} algorithm as the base of our construction to handle exploration-exploitation because of its simplicity. Our approach can also be combined with the other PAC algorithms \citep{strehl2008analysis, dann2015sample} for improved dependence on $S$, $A$ and $H$.

\section{Cover-Rmax}
\label{sec:covrmax}
In this section, we present the Cover-Rmax algorithm and provide a PAC bound for it under a smoothness assumption. The key motivation for our contextual setting is that sharing information among different contexts is helpful. Therefore, it is natural to assume that the MDPs corresponding to similar contexts will themselves be similar. This can be formalized by the following smoothness assumption:
\begin{definition}[Smoothness]
\label{def:smcmdp}
Given a CMDP ($\C{C}, \C{S}, \C{A}, \C{M}$), and a distance metric over the context space $\dis{\cdot}{\cdot}$, if for any two contexts $c_1$, $c_2 \in \C{C}$, we have the following constraints:
\begin{align*}
    \|p^{c_1}(\cdot | s,a) - p^{c_2}( \cdot | s,a)\|_1 \leq L_p \dis{c_1}{c_2}\\
    |r^{c_1}(s,a) - r^{c_2}(s,a)| \leq L_r \dis{c_1}{c_2}
\end{align*}
then, we call this a smooth CMDP with smoothness parameters $L_p$ and $L_r$.
\end{definition}
We assume that the distance metric and the constants $L_p$ and $L_r$ are known. This smoothness assumption allows us to use a minimally tweaked version of Rmax \citep{brafman2002r} and provide an analysis for smooth CMDPs similar to existing literature on MDPs\citep{kearns2002near,strehl2009reinforcement, strehl2008analysis}. If we know the transition dynamics and the expected reward functions for each state-action pair in a finite MDP, we can easily compute the optimal policy. The idea of Rmax is to distinguish the state-action pairs as \emph{known} or \emph{unknown}: 
%by looking at the observed trajectories. 
a state-action pair is known if we have visited it for enough number of times, so that the empirical estimates of reward and transition probabilities are near-accurate due to sufficient data. A state $s$ becomes \emph{known} when all for all actions $a$ the pairs $(s,a)$ become \emph{known}.
%We learn the model of an MDP by accumulating the counts of next state transitions and taking the average of observed rewards. If the number of such counts is high enough, we would have a near-accurate estimate of the model by concentration of measure. Any such state, which has large counts, is called \emph{known}. 
Rmax then constructs an auxiliary MDP which encourages optimistic behaviour by assigning maximum reward (hence the name Rmax) to the remaining \emph{unknown} states. When we act according to the optimal policy in the auxiliary MDP, one of the following must happen: 1) we exploit the information available and achieve near-optimal value, or, 2) we visit unknown states and accumulate more information efficiently. 

Formally, for a set of known states $K$, we define an (approximate) \emph{induced MDP} $\hat{M}_K$ in the following manner. Let $n(s,a)$ and $n(s,a,s')$ denote the number of observations of state-action pair $(s,a)$ and transitions $(s,a,s')$ respectively. Also, let $R(s,a)$ denote the total reward obtained from state-action pair $(s,a)$. For each $s \in K$, define the values
\begin{align}
\label{eq:emp_freq}
\begin{split}
    & p_{\hat{M}_K}(s'|s,a) =  \frac{n(s,a,s')}{n(s,a)}, \\
    &  r_{\hat{M}_K}(s,a) = R(s,a) / n(s,a).
\end{split}
\end{align}
For each $s \notin K$, define the values as $p_{\hat{M}_K}(s'|s,a) = \B{I}\{s' = s\}$ and $r_{\hat{M}_K}(s,a) = 1$.

We use the certainty equivalent policy computed for this induced MDP and perform balanced wandering \citep{kearns2002near} for unknown states. Balanced wandering ensures that all actions are tried equally and fairly for \emph{unknown} states. Assigning maximum reward to the \emph{unknown} states pushes the agent to visit these states and provides the necessary exploration impetus. The generic template of Rmax is given in Algorithm~\ref{alg:Rmax}.

\begin{small}
\IncMargin{2em}
\begin{algorithm}[ht]
    \label{alg:Rmax}
    \SetAlgoNoLine
    % \KwIn{$\C{S}, \C{A}, \C{C}, \epsilon, \delta$ }
    $Initialize(S,A,\C{C}, \epsilon, \delta)$\;
    \For{each episode $t = 1,2,\cdots$}{
        Receive context $c_t \in \C{C}$\;
        Set $K$, $M_K$ using $Predict(c_t,s,a)$ for all $(s,a)$. $\pi \leftarrow \pi^*_{M_K}$\;
        % $j \leftarrow i$ where $c_t \in \C{B}_i$\;
        \For{$h=0,1,\cdots H-1$}{
            \eIf{$s_h \in K$}{
                Choose $a_h := \pi_h(s_h)$\;
            }{ \label{lin:balance_wander}
                Choose $a_h$ : $(s_h,a_h)$ is \emph{unknown}\;
                $Update(c_t, s_h, a_h, (s_{h+1}, r_h))$\;
            }
        }
    }
    \caption{Rmax Template for CMDP}
\end{algorithm}
\DecMargin{2em}
\end{small}

For the contextual case, we have an infinite number of such MDPs. The idea behind our algorithm %given for learning in a smooth contextual MDP 
is that, one can group close enough contexts and treat them as a single MDP. %By tuning the definition of closeness (in terms of a radius parameter), we can control the bias due to ignoring the contexts within each of the groups.
%the error in the resulting MDP's parameters can be bounded too.
Utilizing the boundedness of the context space $\C{C}$, we create a \emph{cover} of $\C{C}$ with finitely many balls $B_r(o_i)$ of radius $r$ centered at $o_i \in \R^d$. By tuning the radius $r$, we can control the bias introduced by ignoring the differences among the MDPs in the same ball.
%then $\forall c \in B_r(o_i)$, the MDP parameters can lie within some $\epsilon$ distance of some contextual MDP in $B_r(o_i)$. 
Doing so allows us to pool together the data from all MDPs in a ball, so that we avoid the difficulty of infinite MDPs and instead only deal with finitely many of them. %and apply the learned policy to any future MDP that falls in this ball. 
The size of the cover, i.e., the number of balls can be measured by the notion of \emph{covering numbers} (see e.g., \cite{cov_num}), defined as 
\begin{align*}
\C{N}(\C{C},r) = \min\{|\C{Y}|\,: \, \C{C} \subseteq \cup_{y \in \C{Y}} \,\, B_r(y)\}.
\end{align*} 
%Given such a cover, we treat each ball as one single MDP and accumulate the transitions to create a representative MDP, hence, arriving at the 
The resulting algorithm, Cover-Rmax, is obtained by using the subroutines in Algorithm~\ref{alg:cover_rmax}, and we state its sample complexity guarantee in Theorem~\ref{thm:pac}.
%For measuring the size of the context space after the covering, we use the notion of covering numbers. 

\begin{small}
\IncMargin{2em}
\begin{algorithm}[!htpb]
    \label{alg:cover_rmax}
    \SetAlgoNoLine
    % \KwIn{$\C{S}, \C{A}, \C{C}, \epsilon, \delta$}
    \Fn{Initialize($\C{S}, \C{A}, \C{C}, \epsilon, \delta$)}{
        $r_0 = \min (\tfrac{\epsilon}{8HL_p},\tfrac{\epsilon}{8L_r})$\;
        Create an $r_0$-cover of $\C{C}$\;
        Initialize counts for all balls $\C{B}(o_i)$\;
    }
    \Fn{Predict($c, s, a$)}{
        Find $j$ such that $c \in \C{B}(o_j)$\;
        \eIf{$n_j(s,a) < m$}{
            \Return $\hat{p}^c(\cdot|s,a)$ and $\hat{r}^c(s,a)$ using (\ref{eq:emp_freq})\;
        }{
            \Return \emph{unknown}\;
        }
    }
    \Fn{Update($c, s, a, (s',r)$)}{
        Find $j$ such that $c \in \C{B}(o_j)$\;
        \If{$n_j(s,a) < m$}{
            Increment counts and rewards in $\C{B}(o_j)$\;
        }
    }
    \caption{Cover-Rmax}
\end{algorithm}
\DecMargin{2em}
\end{small}
\begin{theorem}[PAC bound for Cover-Rmax]
\label{thm:pac}
For any input values $0<\epsilon, \delta \leq 1$ and a CMDP with smoothness parameters $L_p$ and $L_r$, with probability at least $1-\delta$, the Cover-Rmax algorithm produces a sequence of policies $\{\pi_t\}$ which yield at most
\begin{align*}
    \C{O}\Big(\frac{NH^2SA}{\epsilon^3} \big(S + \ln {\frac{NSA}{\delta}} \ln{\frac{N}{\delta}}\big)\Big)
\end{align*}
non-$\epsilon$-optimal episodes, where $N=\C{N}(\C{C},r_0)$ and $r_0 = \min (\tfrac{\epsilon}{8HL_p}, \tfrac{\epsilon}{8L_r})$.
\end{theorem}
\begin{proof}[Proof sketch]
We first of all carefully adapt the analysis of Rmax by \cite{kakade2003sample} to get the PAC bound for an episodic MDP. Let $m$ be the number of visits to a state-action pair after which the model's estimate $\hat{p}(\cdot|s,a)$ for $p(\cdot|s,a)$ has an $\ell_1$ error of at most $\epsilon/4H$ and reward estimate $\hat{r}(s,a)$ has an absolute error of at most $\epsilon/4$. We can show that:
%\nan{I don't understand this sentence: ``We show that if visiting each state action pair $m$ times, guarantees the value obtained in the induced MDP for any policy, to be within $\epsilon/2$ of the value in the true MDP, then the following result holds.'' Is it describing the lemma? Also, lemma inside proof seems weird.}
\begin{lemma}[]
\label{lem:com}
Let $M$ be an MDP with the fixed horizon $H$. If $\hat{\pi}$ is the optimal policy for $\hat{M}_K$ as computed by Rmax, then for any starting state $s_0$, with probability at least $1-2\delta$, we have 
$V^{\hat{\pi}}_{M} \geq V^*_M - 2\epsilon$ 
for all but $\C{O} (\frac{mSA}{\epsilon} \ln{\frac{1}{\delta}})$ episodes.
\end{lemma}
Now instead of learning the model for each contextual MDP separately, we combine the data within each ball. Therefore, we have to take care of two things: choose the radius $r_0$ for a fine enough cover and a value of $m$ which is the number of visits after which a state becomes \emph{known} for a ball. For satisfying the conditions of Lemma~\ref{lem:com} for all MDPs within a ball, we need the radius $r$ of the cover to be $r \leq r_0 = \min \big(\tfrac{\epsilon}{8HL_p}, \tfrac{\epsilon}{8L_r}\big)$ and the value of $m = \frac{128(S\ln 2+\ln{\frac{SA}{\delta})H^2}}{\epsilon^2}$. Using Lemma~\ref{lem:com}, we obtain an upper bound on number of non-$\epsilon$ episodes in a single ball as $\C{O}\big(\frac{H^2SA}{\epsilon^3} \big(S + \ln {\frac{SA}{\delta}} \ln{\frac{1}{\delta}}\big)\big)$ with probability at least $1-\delta$. 

Setting the individual failure probability to be $\delta/N(\C{C},r_0)$ and using the union bound, we get the stated PAC bound.\footnote{For detailed proofs, we refer the reader to the appendix.} 
\end{proof}
%Note that, if we consider total expected reward in an episode instead of average reward per time-step, the bound would have a factor of $H^5$ instead of $H^2$ as we need to set average reward accuracy to $\epsilon/H$ in the $O(1/\epsilon^3)$ dependence.\\
We observe that the PAC bound has linear dependence on the covering number of the context space. In case of a $d$-dimensional Euclidean metric space, the covering number would be of the order $O(\frac{1}{r^d})$. However, we show in Section~\ref{sec:lowbnd}, that, the dependence would be at least linear, and hence, indicate the difficulty of optimally learning in such cases.

\subsection{Lower Bound}
\label{sec:lowbnd}
We prove a lower bound on the number of sub-optimal episodes for any learning algorithm in a smooth CMDP which shows that a linear dependence on the covering number of the context space is unavoidable. As far as we know, there is no existing way of constructing PAC lower bounds for continuous state spaces with smoothness, so we cannot simply augment the state representation to include context information. Instead, we prove our own lower bound in Theorem~\ref{thm:smooth_lower} which builds upon the work of \cite{dann2015sample} on lower bounds for episodic finite MDPs and of \cite{slivkins2014contextual} on lower bounds for contextual bandits. %We can then use the construction to show the following lower bound:
\begin{theorem}[Lower bound for smooth CMDP] \label{thm:smooth_lower}
There exists constants $\delta_0, \epsilon_0$, such that for every $\delta \in (0,\delta_0)$ and $\epsilon \in (0,\epsilon_0)$, any algorithm that satisfies a PAC guarantee for $(\epsilon,\delta)$ and computes a sequence of deterministic policies for each context, there is a hard CMDP $(\C{C}, \C{S}, \C{A}, \C{M})$ with smoothness  constant $L_p = 1$, %, L_r = 1$ 
such that
\begin{align}
\B{E}[B] = \Omega\Big( \tfrac{\C{N}(\C{C},\epsilon_1)SA}{\epsilon^2}\Big)
\end{align}
where $B$ is the number of sub-optimal episodes and  $\epsilon_1 = \tfrac{1280 H \epsilon e^4}{(H-2)}$ . %until the algorithm's policy is $(\epsilon,\delta)$-accurate.
\label{thm:cmdp_low}
\end{theorem}
%\nan{TODO: Adapt the lower bounds to our normalized setting and drop the $H^2$ dependence (Aditya will verify). Shorten this clarification paragraph and move it to footnote.}
%\textbf{Note:} Here, the value function assumed is the total expected reward achieved in an episode instead of the average reward per timestep.
%The lower bound clearly shows that the linear dependence on the covering number is unavoidable in the PAC bound for a CMDP with just smoothness assumptions. %However, we note that the the radius does not depend inversely on the horizon length $H$.
\begin{figure}
    \centering
    % \vspace{.3in}
    \includegraphics[width=0.75\columnwidth]{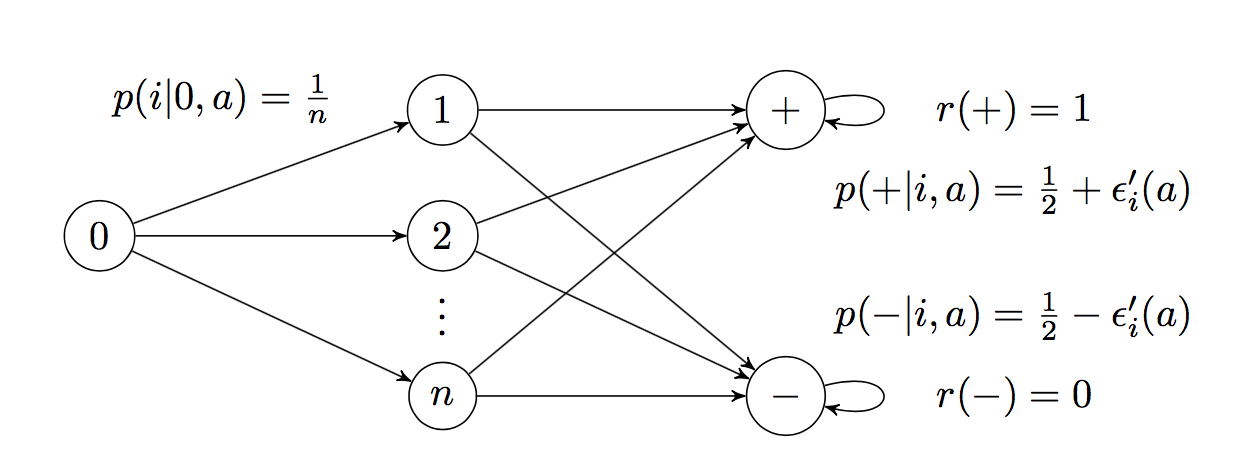}
    % \vspace{.3in}
    \caption{Hard instances for episodic MDP \citep{dann2015sample}. The initial state $0$ moves to a uniform distribution over states $1$ to $n$ regardless of the action, and states $+$/$-$ are absorbing with $1$ and $0$ rewards respectively. States $0$ to $n$ have $0$ reward for all actions. Each  state $i \in [n]$ essentially acts as a hard bandit instance, whose $A$ actions move to $+$ and $-$ randomly. Action $a_0$ satisfies $p(+|i,a_0) = \tfrac{1}{2}+\tfrac{\epsilon'}{2}$ and there is at most one other action $a_i$ with $p(+|i,a_i) = \tfrac{1}{2}+\epsilon'$. Any other action $a_j$ satisfies $p(+|i, a_j) =  \tfrac{1}{2}$.}
    \label{fig:lowbnd}
\end{figure}
\begin{proof}
%Consider a packing of the space $\C{C}$ with a radius $r$, defined as a set of points $X$ such that any two points in $X$ lie at a distance of at least $r$. The packing number of $\C{C}$ is:
%\begin{align*}
%\C{D}(\C{C},r) = \max \{|X|:\, X \text{ is an } r \text{-packing of }\C{C}\}
%\end{align*}
%We will put hard MDP instances on each of these points. When $r$ is sufficiently large, the MDPs at different points in $X$ are essentially independent, and the agent cannot infer anything about the MDP in one point from data about the MDPs at other packing points.\\
The overall idea is to embed multiple MDP learning problems in a CMDP, such that the agent has to learn the optimal policy in each MDP separately and cannot generalize across them. We show that the maximum number of problems that can be embedded scales with the covering number, and the result follows by incorporating known PAC lower bound for episodic MDPs.

We start with the lower bound for learning in episodic MDPs. See Figure~\ref{fig:lowbnd} and its caption for details. The construction is due to \cite{dann2015sample} and we adapt their lower bound statement to our setting in Theorem~\ref{thm:mdp_low}.  %Therefore, the agent starts off at $0$ and moves to a next state, then depending on the chosen action, it collects the reward at $+/-$ for remaining $H-2$ steps. 
\begin{theorem}[ Lower bound for episodic MDPs \citep{dann2015sample} ]
\label{thm:mdp_low}
There exists constants $\delta_0, \epsilon_0$, such that for every $\delta \in (0,\delta_0)$ and $\epsilon \in (0,\epsilon_0)$, any algorithm that satisfies a PAC guarantee for $(\epsilon,\delta)$ and computes a sequence of deterministic policies, there is a hard instance $M_{\text{hard}}$ so that
$
\B{E}[B] = \Omega\Big( \tfrac{SA}{\epsilon^2}\Big)
$, 
where $B$ is the number of sub-optimal episodes. % until the algorithm's policy is $(\epsilon,\delta)$-accurate. 
The constants can be chosen as $\delta_0 = \frac{e^{-4}}{80}$, $\epsilon_0 = \frac{H-2}{640 H e^4}$.\footnote{The lower bound here differs from that in the original paper by $H^2$, because our value is normalized (see Eq.(\ref{eq:value})), whereas they allow the magnitude of value to grow with $H$.}%The lower bound is adapted for average reward which reduces the value of $\epsilon'$ in the construction by a factor of $H$. This is easy to see as the required reward difference for a sub-optimal arm should be decreased to accomodate the normalization in value function.}
\end{theorem}
Now we discuss how to populate the context space with these hard MDPs. Note in Figure~\ref{fig:lowbnd} that, the agent does not know which action is the most rewarding ($a_i$), and the adversary can choose $i$ to be any element of $[A]$ (which is essentially choosing an instance from a family of MDPs). In our scenario, we would like to allow the adversary to choose the MDP \emph{independently} for each individual packing point to yield a lower bound linear in the packing number. However, this is not always possible due to the smoothness assumption, as committing to an MDP at one point may restrict the adversary's choices in another point.

To deal with this difficulty, we note that any pair of hard MDPs  differ from each other by $O(\epsilon')$ in transition distributions. Therefore, we construct a packing of $\C{C}$ with radius $r = 8\epsilon'$, defined as a set of points $Z$ such that any two points in $Z$ are at least $r$ away from each other. The maximum size of such $Z$ is known as the \emph{packing number}:
\begin{align*}
\C{D}(\C{C},r) = \max \{|Z|:\, Z \text{ is an } r\text{-packing of }\C{C}\},
\end{align*}
which is related to the covering number as $N(\C{C}, r) \le \C{D}(\C{C},r)$. 
The radius $r$ is chosen to be $O(\epsilon')$ so that arbitrary choices of hard MDP instances at different packing points always satisfy the smoothness assumption (recall that $L_p=1$).  %For the transition probabilities, without loss of generality, set $L_p$ to 1. For each point in $Z$, we will assign a hard instance with parameter $\epsilon'$ later. Assuming such an assignment, 
Once we fix the MDPs for all $c\in Z$, we specify the MDP for $c \in \C{C} \setminus Z$ as follows: for state $i$ and action $a$,
\begin{align*}
p_c(+|i,a) = \max_{c' \in Z} \max (1/2, p_{c'}(+|i,a) - \dis{c}{c'}/2).
\end{align*}
Essentially, as we move away from a packing point, the transition to $+$/$-$ become more uniform. We can show that:
\begin{claim}
\label{cl:val_SCMDP}
The CMDP defined above is satisfies Definition~\ref{def:smcmdp} with constant $L_p = 1$.\footnote{The reward function does not vary with context hence reward smoothness is satisfied for all $L_r\ge 0$. The proof of the claim is deferred to the appendix.} 
\end{claim}
%\nan{Agreed with Ambuj. I have commented out a big chunk of text below. Aditya: please move things to appendix as you feel necessary.}
%For every $c \in Z$, we choose the optimal action for each state $i \in [n]$, uniformly among all actions. This choice is independent for each hard instance. 
%Because of this independent choice, we have:
%\begin{align*}
%    \B{E}_{\text{all}}[B_i] = \B{E}_{\text{all}}[B_i|B_1,\ldots,B_{i-1},B_{i+1},\ldots,B_{\C{D}(\C{C},8\epsilon')}]
%\end{align*}
%where $\B{E}_{\text{all}}$ includes the randomness in the instances due to this action choice. \ambuj{I am not sure if anyone will understand this conditioning stuff at this high level... I can't follow it myself. Shouldn't we do a more careful conditioning somewhere in the appendix with different sources of randomness made explicit in the notation?} This is true because the data streams generated by each hard instance in the packing is statistically independent of the other points. Thus, the problem behaves like an individual instance of a hard MDP in this case too, and we can apply Theorem~\ref{thm:mdp_low}. In other words, full or partial information about one does not give us any extra information about the other contexts. The choice of $\epsilon'$ is taken as $\tfrac{160\epsilon e^4}{(H-2)}$ which is required by the construction of theorem~\ref{thm:mdp_low}.  
We choose the context sequence given as input to be repetitions of an arbitrary permutation of $Z$. By construction, the learning at different points in $Z$ are independent, so the lower bound is simply the lower bound for learning a single MDP (Theorem~\ref{thm:mdp_low}) multiplied by the cardinality of $Z$ (the packing number). %By Theorem \ref{thm:mdp_low}, we know that the expected number of episodes required by each instance before becoming $(\epsilon,\delta)$ optimal is lower bounded. Using linearity of expectation, we get the result that
%\begin{align*}
 %   \B{E}_{\text{all}}[B] \geq \Omega\Big( \tfrac{\C{D}(\C{C},\epsilon_1)SAH^2}{\epsilon^2} \ln(\tfrac{c_1}{\delta+c_2})\Big)
%\end{align*}
%Since, this result holds in expectation over the distribution over all instances, there exists at least one fixed instance of smooth CMDP in which the expected number of sub-optimal episodes is greater than the stated quantity.
Using the well known relation that $\C{N}(\C{C},r) \leq \C{D}(\C{C},r)$, we have the desired lower bound. We refer the reader to the appendix for proof of Claim~\ref{cl:val_SCMDP} and a more detailed analysis.
\end{proof}

\section{Contextual Linear Combination of MDPs} \label{sec:linear}
From the previous section, it is clear that for a contextual MDP with just smoothness assumptions, exponential dependence on context dimension is unavoidable. Further, the computational requirements of our Cover-Rmax algorithm scales with the covering number of the context space. As such, in this section, we focus on a more structured assumption about the mapping from context space to MDPs and show that we can achieve substantially improved sample efficiency.

The specific assumption we make in this section is that the model parameters of an individual MDP $M^c$ is the linear combination of the parameters of $d$ base MDPs, i.e., 
\begin{align} \label{eq:lin_estimate}
\begin{split}
    & p^c(s'|s,a) = c^\top \begin{bmatrix} p_1(s'|s,a) \\
    \vdots\\
    p_d(s'|s,a)
    \end{bmatrix} := c^\top P(s,a,s'), \\
    & r^c(s,a) = c^\top \begin{bmatrix} r_1(s,a) \\
    \vdots\\
    r_d(s,a) %\label{eq:lin:reward}
    \end{bmatrix}  := c^\top R(s,a).
\end{split}
\end{align}
We use $P(s,a,s')$ and $R(s,a)$ as shorthand for the $d\times 1$ vectors that concatenate the parameters from different base MDPs for the same $s,a$ (and $s'$). The parameters of the base MDPs ($p_i$ and $r_i$) are unknown and need to be recovered from data by the learning agent, and the combination coefficients are directly available which is the context vector $c$ itself. This assumption can be motivated  in an application scenario where the user/patient responds according to her characteristic distribution over $d$ possible behavioural patterns.

A mathematical difficulty here is that for an arbitrary context vector $c \in \R^d$, $p^c(\cdot|\cdot,\cdot)$ is not always a valid transition function and may violate non-negativity and normalization constraints. Therefore, we require that $c \in \Delta_{d-1}$, that is, $c$ stays in the probability simplex so that $p^c(\cdot|\cdot,\cdot)$ is always valid.\footnote{$\Delta_n$  is the $n$-simplex $\{x \in \R^{n+1}: \, \sum_{i=1}^{n+1} x_i = 1, x \geq 0$\}.} 

\subsection{KWIK\_LR-Rmax}
\label{sec:kwikrmax}
We first explain how to estimate the model parameters in this linear setting, and then discuss how to perform exploration properly. 

\paragraph{Model estimation} 
Recall that in Section~\ref{sec:covrmax}, we treat the MDPs whose contexts fall in a small ball as a single MDP, and estimate its parameters using data from the \emph{local} context ball. In this section, however, we have a \emph{global} structure due to our parametric assumption ($d$ base MDPs that are shared across all contexts). This implies that data obtained at a context may be useful for learning the MDP parameters at another context that is far away, and to avoid the exponential dependence on $d$ we need to leverage this structure and generalize globally across the entire context space. 

Due to the linear combination setup, we use linear regression to replace the estimation procedure in Equation~\ref{eq:emp_freq}: %$p^c(s,a,s')$ denote the transition probability $p^c(s'|s,a)$ for a context $c$, and,
in an episode with context $c$, when we observe the state-action pair $(s,a)$, a next-state $s_{\text{next}}$ will be drawn from $p^c(\cdot|s,a)$.\footnote{Here we use $s_{\text{next}}$ to denote the random variable, and $s'$ to denote a possible realization.} Therefore, the indicator of whether $s_{\text{next}}$ is equal to $s'$ forms an unbiased estimate of $p^c(s'|s,a)$, i.e.,
$
\B{E}_{s_{\text{next}} \sim p^c(\cdot|s,a)}\left[\mathbb{I}[s_{\text{next}}=s'] \right]= p^c(s'|s,a) = c^\top P(s,a,s').
$ 
Based on this observation, we can construct a feature-label pair 
\begin{align} \label{eq:input-output}
    (c,\, \mathbb{I}[s_{\text{next}}=s'])
\end{align}
whenever we observe a transition tuple $(s,a,s_{\text{next}})$ under context $c$, and their relationship is governed by a linear prediction rule with $P(s,a, s')$ being the coefficients. Hence, to estimate $P(s,a,s')$ from data, we can simply collect the feature-label pairs that correspond to this particular $(s,a,s')$ tuple, and run linear regression to recover the coefficients. The case for reward function is similar, hence, not discussed. 

If the data is abundant (i.e., $(s,a)$ is observed many times) and exploratory (i.e., the design matrix that consists of the $c$ vectors for $(s,a)$ is well-conditioned), we can expect to recover $P(s,a,s')$ accurately. But how to guarantee these conditions? Since the context is chosen adversarially, the design matrix can indeed be ill-conditioned.

%Our resolutions to these difficulties are based on the following intuitions:
Observe, however, when the matrix is ill-conditioned and new contexts lie in the subspace spanned by previously observed contexts, we can make accurate predictions despite the inability to recover the model parameters. An \emph{online} linear regression (LR) procedure will take care of this issue, and we choose KWIK\_LR \citep{walsh2009exploring} as such a procedure. 

The original KWIK\_LR deals with scalar labels, which can be used to decide whether the estimate of $p^c(s'|s, a)$ is sufficiently accurate (\emph{known}). A $(s,a)$ pair then becomes known if $(s,a,s')$ is known for all $s'$. This approach, however, generally leads to a loose analysis, because there is no need to predict $p^c(s'|s,a)$ for each individual $s'$ accurately: if the estimate of $p^c(\cdot |s, a)$ is close to the true distribution under $L_1$ error, the $(s,a)$ pair can already be considered as known. We extend the KWIK\_LR analysis to handle vector-valued outputs, and provide tighter error bounds by treating $p^c(\cdot|s,a)$ as a whole. Below we introduce our extended version of KWIK\_LR, and explain how to incorporate the knownness information in Rmax skeleton to perform efficient exploration.

%$P(s,a,s')$ to denote the vector $[p_1(s'|s,a), p_2(s'|s,a),\ldots,p_d(s'|s,a)]^\top$. If we stack these vectors for all next states $s'$, we get a singly stochastic matrix of shape $d \times S$, which includes all the unknown parameters for the transition distribution at $(s,a)$ and we denote it by $P(s,a)$. 

%Similarly, let $R(s,a)$ denote the vector of rewards for the state-action pair $(s,a)$ in the base MDPs.
%For any new context, we now have to know the transition probabilities $p^c(\cdot|s,a) = [P(s,a)]^T c$ and rewards $r(s,a) = [R(s,a)]^T c$.

\paragraph{Identifying known $(s,a)$ with KWIK\_LR}~\\
%For an efficient way of exploration in this \emph{structured} contextual MDP, lets revisit the Rmax skeleton. \ambuj{I think "structured" has a different connotation when used with MDPs. We need to be careful.} With the smoothness assumption, we argued that we can treat the MDPs whose contexts fall in a small ball as a single MDP. Consequently, we used visitation counts to determine whether a state-action pair is known or unknown for a \emph{local} context ball, and assigned optimistic value to unknown states to drive exploration. 
%
%With this covering intuition, we used the counts for state-action pairs to obtain a good approximation of the MDP parameters. 
%
%Since the local visitation counts no longer capture the notion of known vs unknown appropriately in this section, we need to find an alternative that can tell us whether our current estimation of $p^c(\cdot|s,a)$ and $r^c(\cdot |s, a)$.
%
%For every contextual MDP, at the beginning of each episode, we build the induced MDP $M_K$. Therefore, we need an oracle which tells us which state-action parameters ($p(\cdot|s,a),r(s,a)$) are known nearly accurately. We use the KWIK (``know what it knows'') linear regression algorithm proposed by Walsh et al.~as this oracle \cite{walsh2009exploring}.
%
The KWIK\_LR-Rmax algorithm we propose for the linear setting still uses Rmax template (Algorithm~\ref{alg:Rmax}) for exploration: in every episode we build the induced MDP $M_K$, and act greedily according to its optimal policy with balanced wandering. %The template of the KWIK-Rmax setup can be seen in Algorithm~\ref{alg:KWIKrmax}.
The major difference from Cover-Rmax lies in how the set of known states $K$ are identified and how $M_K$ is constructed, which we explain below (see pseudocode in Algorithm~\ref{alg:KWIK}).

At a high level, the algorithm works in the following way: when constructing $M_K$, we query the KWIK procedure for estimates $\hat{p}^c(\cdot|s,a)$ and $\hat{r}^c(s,a)$ for every pair $(s,a)$ using $Predict(c,s,a)$. The KWIK procedure either returns $\perp$ (don't know), or returns estimates that are guaranteed to be accurate. 
%\begin{align*}
%    & \|\hat{p}^c(\cdot|s,a) - p^c(\cdot|s,a)\|_1 \leq \epsilon/4H, \\
%    & |\hat{r}^c(s,a) - r^c(s,a)| \leq \epsilon/4,
%\end{align*}
%with probability at least $1-\delta'$. 
If $\perp$ is returned, then we consider $(s,a)$ as unknown and associate $s$ with $R_{\max}$ reward for exploration. %prediction is when we use optimistic exploration and assign the $R_{\text{max}}$ reward to $(s,a)$. 
%The quantity of interest, therefore, is the KWIK bound which gives the maximum number of such predictions over the context space.
Such optimistic exploration ensures significant probability of observing $(s,a)$ pairs on which we have predicted $\perp$. If we do observe such pairs in an episode, we call $Update$ with feature-label pairs formed via Equation~\ref{eq:input-output} to make  progress on estimating parameters for unknown state-action pairs. 

Next we walk through the pseudocode and explain how $Predict$ and $Update$ work in detail. Then we prove an upper bound on the number of updates that can happen (i.e., the \textbf{if} condition holds on Line \ref{lin:updateQ}), which forms the basis of our analysis of KWIK\_LR-Rmax. 

In Algorithm~\ref{alg:KWIK}, we initialize matrices $Q$ and $W$ for each $(s,a)$ using $Initialize(\cdot)$ and update them over time. Let $C_t(s,a)$ be the design matrix at episode $t$, where each row is a context $c_\tau$ such that $(s,a)$ was observed in episode $\tau < t$. By matrix inverse rules, we can verify that the update rule on  Line~\ref{lin:update_Q} essentially yields $Q_t(s,a) = (I + C_t^\top C_t)^{-1}$, where $Q_t(s,a)$ is the value of $Q(s,a)$ in episode $t$. This is the inverse of the (unnormalized and regularized) empirical covariance matrix, which plays a central role in linear regression analysis. The matrix $W$ accumulates the outer product between the feature vector (context) $c$ and the one-hot vector label 
$y = (\{\B{I}[s_{\text{next}} = s']\}_{\forall s' \in \C{S}})^\top$. 
%$y_h = \{\B{I}[s_{h+1} = s']\}_{\forall s' \in \C{S}}$. 
It is then obvious that $Q_t(s,a) W_t(s,a)$ is the linear regression estimate of $P(s,a)$ using the data up to episode $t$.

%The formulation of the KWIK linear regression algorithm can be shown in the following manner. Let the contexts observed at each time-step be denoted by $c_1, c_2, \ldots, c_t \in \Delta_{d-1}$. We denote the context matrix observed till now as $C_t = [c_1, c_2,\ldots, c_t] \in \R^{d \times t}$. For each $j = 1,\ldots,t$, the value $p^{c_j}(\cdot|s,a)$ is given by $[P(s,a)]^T{c_j}$. Now, the solution to the equation:
%\begin{align*}
%    \begin{bmatrix}I\\C_t^T \end{bmatrix} v = \begin{bmatrix} P(s,a)\\
%    p_t(\cdot|s,a)
%    \end{bmatrix}
%\end{align*}
%is unique and is $v = P(s,a)$. 
%Denoting the matrix $[I;C_t^T]$ as $A_t$, we can write this as:
%\begin{eqnarray*}
%   v &=& Q_tA_t^T \begin{bmatrix} P(s,a)\\
%    p_t(\cdot|s,a)
%    \end{bmatrix}
%\end{eqnarray*}
%where $Q_t = (A_t^T A_t)^{-1}$. We do not know $P(s,a)$ and only observe samples $y^{c}(s,a) = \{\B{I}[s_{\text{next}} = s']\}_{\forall s' \in \C{S}}$ which is in $\{0,1\}^{1 \times S}$, instead of $p^{c}(\cdot|s,a)$. Thus, we write the estimate as:
%\begin{align*}
%    \hat{v} = Q_t A_t^T \begin{bmatrix}
%           0\\
%           y_t(s,a)
%         \end{bmatrix}
%\end{align*}
When a new input vector $c_t$ comes, we check whether $\|Q(s,a) c_t\|_2$ is below a predetermined threshold $\alpha_S$ (Line~\ref{lin:kwik_known}). Recall that $Q(s,a)$ is the inverse covariance matrix, so a small $\|Q(s,a) c_t\|_2$ implies that the estimate $Q_t(s,a) W_t(s,a)$ is close to $P(s,a)$ along the direction of $c_t$, so we predict $p^{c_t}(\cdot |s,a) = c_t^\top P(s,a) \approx c_t^\top Q(s,a) W(s,a)$; otherwise we return $\perp$. %After every update, we make learning progress as appending $c_t$ to $C_t$ makes the design matrix more well-conditioned. 
%If the condition holds, we estimate the mean output $f(c)$ as $c^T Q_t C_t y_t $, otherwise, output $\perp$. 
The KWIK subroutine for rewards is similar hence omitted. To ensure that the estimated transition probability is valid, we project the estimated vector onto $\Delta_{S-1}$, which can be done efficiently using existing techniques \citep{duchi2008efficient}. 

Below we state the KWIK bound for learning the transition function; the KWIK bound for learning rewards is much smaller hence omitted here. 
%We extend the KWIK analysis in \cite{walsh2009exploring} by combining it with concentration inequalities for vector valued martingales and properties of stochastic vectors.
We use the KWIK bound for scalar linear regression from \cite{walsh2009exploring} and the property of multinomial samples to get our KWIK bound. 
\begin{theorem}[KWIK\_LR bound for learning multinomial vectors]
\label{thm:kwiklr}
For any $\epsilon > 0$ and $\delta > 0$, if the KWIK\_LR algorithm is executed for probability vectors $p_t(\cdot|s,a)$, with $\alpha_S = \min \{b_1\frac{\epsilon^2}{d^{3/2}}, b_2\tfrac{ \epsilon^2}{\sqrt{d}\log (d 2^S /\delta)}, \tfrac{\epsilon}{2\sqrt{d}}\}$ with suitable constants $b_1$ and $b_2$, then the number of $\perp$'s where updates take place (see Line~\ref{lin:updateQ}) will be bounded by $\C{O}(\tfrac{d^2}{\epsilon^4} \max \{d^2, S^2 \log^2 (d/\delta')\})$, and, with probability at least $1-\delta$, $\forall c_t$ where a non-``$\perp$'' prediction is returned, $\|\hat{p}^{c_t}_t(\cdot|s,a) - p^{c_t}_t(\cdot|s,a)\|_1 \leq \epsilon$.
\end{theorem}
\begin{proof}[Proof sketch]
(See full proof in the appendix.) We provide a direct reduction to KWIK bound for learning scalar values. The key idea is to notice that for any vector $v\in \R^{S}$: 
$$
\|v\|_1 = \sup_{f\in \{-1, 1\}^{S}} v^\top f.
$$
So conceptually we can view Algorithm~\ref{alg:KWIK} as running $2^{S}$ scalar linear regression simultaneously, each of which projects the vector label to a scalar by a fixed linear transformation $f$. We require every scalar regressor to have $(\epsilon, \delta/2^{S})$ KWIK guarantee, and the $\ell_1$ error guarantee for the vector label follows from union bound.
\end{proof}
% \begin{small}
% \IncMargin{2em}
% \begin{algorithm}
%     \SetAlgoNoLine
%     \KwIn{$\C{S}, \C{A}, d, 0 < \epsilon, \delta \leq 1, \alpha_S$ \nan{Call them parameters?}\aditya{TODO}}
% %    \For{all $(s,a) \in \C{S} \times \C{A}$}{
% \Fn{Initialize()}{
%     $\forall s \in \C{S}, a \in \C{A}$, $Q(s,a) \leftarrow I_d$, $W(s,a) \leftarrow \{0\}^{d \times S}$\;
% }
% %    }
% %   \For{each episode $t = 1,2,\cdots$}{
% \Fn{Predict($c$)}{
% %        Observe context $c_t \in \Delta_{d-1}$\;
% %        $K \leftarrow \emptyset$\;
%         \For{$s \in \C{S}$}{ \label{lin:kwik_known}
%             \eIf{$\forall a \in \C{A},\,\,\|Q(s,a)c\|_1 < \alpha_S $}{
%                 $\hat{p}^{c}(\cdot|s,a) = c^T Q(s,a) W(s,a)$\;
%             }{
%                 $\hat{p}^{c}(\cdot|s,a) = \perp$\;
%             }
%         }
%     \Return $\hat p^{c}$.
% }
% \Fn{Update($c$, $\hat p^c$, $\{s_0, a_0, s_1\ldots, s_H\}$)}{
%     \For{$h=1$ \textbf{to} $H-1$}{ \label{lin:update_Q}
%         \If{$\hat{p}^{c}(\cdot|s_h,a_h) = \perp$}{
%             $Q(s_h,a_h) \leftarrow Q(s_h,a_h) - \tfrac{(Q(s_h,a_h)c)(Q(s_h,a_h)c)^T}{1+c^TQ(s_h,a_h)c}$\label{lin:Q_upd}\;
%             $y_h \leftarrow \{\B{I}[s_{h+1} = s']\}_{\forall s' \in \C{S}} \in \{0,1\}^{1 \times S}$\;
%             $W(s_h,a_h) \leftarrow W(s_h,a_h) + c y_h$\;
%         }
%     }
% }
% %    }
%     \caption{KWIK learning of $p^c(\cdot|s,a)$} 
%     \label{alg:KWIK}
% \end{algorithm}
% \DecMargin{2em}
% \end{small}

\begin{small}
\IncMargin{2em}
\begin{algorithm}
    \SetAlgoNoLine
\Fn{Initialize($S, d, \alpha_S$)}{
    $Q(s,a) \leftarrow I_d$ for all $(s,a)$\;
    $W(s,a) \leftarrow \{0\}^{d \times S}$ for all $(s,a)$\;
}
\Fn{Predict($c,s,a$)}{
        \eIf{$\|Q(s,a)c\|_1 \le \alpha_S $ \label{lin:kwik_known} }{
                \Return $\hat{p}^{c}(\cdot|s, a) = c^\top Q(s,a) W(s,a)$\;
        }{
            \Return $\hat{p}^{c}(\cdot|s, a) = \perp$\;
        }
}
\Fn{Update($c, s, a, s_{\text{next}}$)}{
    \label{lin:update_Q}
    \If{$\|Q(s,a)c\|_1 > \alpha_S $ (``$\perp$'' prediction)}{ \label{lin:updateQ}
        $Q(s,a) \leftarrow Q(s,a) - \tfrac{(Q(s,a)c)(Q(s,a)c)^\top}{1+c^\top Q(s,a)c}$\label{lin:Q_upd}\;
        $y \leftarrow (\{\B{I}[s_{\text{next}} = s']\}_{\forall s' \in \C{S}})^\top$\;
        $W(s,a) \leftarrow W(s,a) + c y$\; \label{lin:updateW}
    }
}
%    }
    \caption{KWIK learning of $p^c(\cdot|s,a)$} 
    \label{alg:KWIK}
\end{algorithm}
\DecMargin{2em}
\end{small}

%Instead of using counts as our exploration method, we now use the KWIK oracle to query the \emph{knownness} of a state. We can easily arrive at the following PAC bound for this case by using the KWIK bound:
With this result, we are ready to prove the formal PAC guarantee for KWIK\_LR-Rmax.
\begin{theorem}[PAC bound for KWIK\_LR-Rmax]
\label{thm:kwikpac}
For any input values $0<\epsilon, \delta \leq 1$ and a linear CMDP model with $d$ number of base MDPs, with probability $1-\delta$, the KWIK\_LR-Rmax algorithm, produces a sequence of policies $\{\pi_t\}$ which yield at most
\begin{align*}
    \C{O}\Big(  \frac{d^2H^4SA}{\epsilon^5} \log \tfrac{1}{\delta} \max\{d^2,S^2 \log^2(\tfrac{dSA}{\delta})\}\Big)
\end{align*}
non-$\epsilon$-optimal episodes.
\end{theorem}
\begin{proof}
When the KWIK subroutine (Algorithm~\ref{alg:KWIK}) makes non-``$\perp$'' predictions $\hat{p}^c(s,a,s')$, we require that
\begin{align*}
    \|\hat{p}^c(\cdot|s,a) - p^c(\cdot|s,a)\|_1 \leq \epsilon/8H.
\end{align*}
After projection onto $\Delta_{S-1}$, we have:
\begin{align*}
    \|\Pi_{\Delta_{S-1}}(\hat{p}^c(s,a)) - p^c(\cdot|s,a)\|_1 \leq  2\|\hat{p}^c(\cdot|s,a) - p^c(\cdot|s,a)\|_1
    \leq \epsilon/4H.
\end{align*}
%The required guarantee for the rewards would hold in this case too. The KWIK oracle either predicts the MDP parameters accurately or outputs $\perp$ which makes the state unknown. 
Further, the update to the matrices $Q$ and $W$ happen only when an unknown state action pair $(s,a)$ is visited and the KWIK subroutine still predicts $\perp$ (Line~\ref{lin:update_Q}). The KWIK bound states that after a fixed number of updates to an unknown $(s,a)$ pair, the parameters will always be known with desired accuracy. The number of updates $m$ can be obtained by setting the desired accuracy in transitions to $\epsilon/8H$ and failure probability as $\delta/SA$ in Theorem \ref{thm:kwiklr}:
\begin{align*}
m = \C{O}\big(  \tfrac{d^2H^4}{\epsilon^4} \max\{d^2,S^2 \log^2(\tfrac{dSA}{\delta})\}\big)
\end{align*}
We now use Lemma~\ref{lem:com} where instead of updating counts for number of visits, we look at the number of updates for unknown $(s,a)$ pairs. On applying a union bound over all state action pairs and using Lemma~\ref{lem:com}, it is easy to see that the sub-optimal episodes are bounded by $\C{O}\Big(\frac{mSA}{\epsilon} \ln \frac{1}{\delta}\Big)$ 
with probability at least $1-\delta$. The bound in Theorem~\ref{thm:kwikpac} is obtained by substituting the value of $m$.
\end{proof}
We see that for this contextual MDP, the linear structure helps us in avoiding the exponential dependence in context dimension $d$. The combined dependence on $S$ and $d$ is now $O(\max \{d^4S, d^2S^3\})$.

\section{Related work} \label{sec:related}
\paragraph{Transfer in RL with latent contexts}
The general definition of CMDPs  captures the problem of transfer in RL and multi-task RL. See \cite{taylor2009transfer} and \cite{lazaric2011transfer} for surveys of empirical results. 
%In the broad definition of Contextual MDPs, we have, essentially, formulated a problem in transfer in RL or multi-task RL. For prior empirical results in these problems, we refer the reader to the survey papers by \cite{taylor2009transfer} and \cite{lazaric2011transfer}.
Recent papers have also advanced the theoretical understanding of transfer in RL. For instance, \cite{brunskill2013sample} and \cite{hallak2015contextual} analyzed the sample complexity of CMDPs %of multi-task RL in a special case of CMDPs where a \emph{latent} (i.e., unseen) context  maps the environment to a finite set of $K$ MDPs. 
where each MDP is an element of a finite and small set of MDPs, and the MDP label is treated as the \emph{latent} (i.e., unseen) context. 
%Their algorithm takes a separability assumption for these $K$ MDPs, which allows a PAC bound independent of the number of states and actions. 
\cite{mahmud2013clustering} consider the problem of 
transferring the optimal policies of a large set of known MDPs to a new MDP. 
%The idea of grouping similar MDPs is a common thread underlying many of the papers in multi-task RL. However, existing literature provides sample complexity results for multi-task RL with finite number of MDPs.
The commonality of the above papers is that the MDP label (i.e., the context) is not observed. Hence, their methods have to initially explore in every new MDP to identify its label, which requires the episode length to be substantially longer than the planning horizon. This can be a problematic assumption in our motivating scenarios, where we interact with a patient / user / student for a limited period of time and the data in a single episode (whose length $H$ \emph{is} the planning horizon) is not enough for identifying the underlying MDP.  %and does not apply to arbitrary horizon lengths.
%Our work considers the setting of continuous and observed side-information and works for any value of $H$.
%\cite{hallak2015contextual} extended the use of contexts to MDPs and proposed the \emph{Contextual MDP} framework. They look at a specific instance with finite and latent contexts similar to the setting of \cite{brunskill2013sample}. As the contexts are latent, their method has to initially explore in every episode to identify the MDP and does not apply to arbitrary horizon lengths. Our work considers the setting of continuous and observed side-information and works for any value of $H$.
In contrast to prior work, we propose to leverage observable context information to perform more direct transfer from previous MDPs, and our algorithm works with arbitrary episode length $H$.

\paragraph{RL with side information} Our work leverages the available side-information for each MDP, which is inspired by the use of contexts in contextual bandits \citep{langford2008epoch, li2010contextual}. %Further, the availability of contexts enables generalization across large and potentially infinite number of MDPs. 
The use of such side information can also be found in RL literature: \cite{ammar2014online} developed a multi-task policy gradient method where the context is used for transferring knowledge between tasks; \cite{killian2016transfer} used parametric forms of MDPs to develop models for personalized medicine policies for HIV treatment.

\paragraph{RL in metric space} For smooth CMDPs (Section~\ref{sec:covrmax}), we pool observations across similar contexts and reduce the problem to learning policies for finitely many MDPs. An alternative approach is to consider an infinite MDP whose state representation is augmented by the context, and apply PAC-MDP methods for metric state spaces (e.g., C-PACE proposed by \cite{pazis2013pac}).
However, doing so might increase the sample and computational complexity unnecessarily, because we no longer leverage the structure that 
a particular component of the (augmented) state, namely the context, remains the same in an episode. Concretely, the augmenting approach needs to perform planning in the augmented MDP over states and contexts, which makes its computational/storage requirement worse than our %cover based 
solution: we only perform planning in MDPs defined on $\C{S}$, whose computational characteristics have no dependence on the context space. 
In addition, we allow the context sequence to be chosen in an adversarial manner. This corresponds to adversarially chosen initial states in MDPs, which is usually not handled by PAC-MDP methods.

\paragraph{KWIK learning of linear hypothesis classes} Our linear combination setting (Section~\ref{sec:linear}) provides an instance where parametric assumptions can lead to substantially improved PAC bounds. We build upon the KWIK-Rmax learning framework developed in previous work \citep{li2008knows, szita2011agnostic} and use KWIK linear regression as a sub-routine. 
%general KWIK based PAC algorithm given by \cite{li2008knows} using Rmax as the skeleton. 
For the resulting KWIK\_LR-Rmax algorithm, its sample complexity bound inherently depends on the KWIK bound for linear regression. It is well known that even for linear hypothesis classes, the KWIK bound is exponential in input dimension in the agnostic case \citep{szita2011agnostic}. Therefore, the success of the algorithm relies on the validity of the modelling assumption.

%In addition to our approach, the linearity assumption can be used to incorporate bandit stochastic linear optimization methods \citep{dani2008stochastic} in interval estimation based exploration methods \citep{strehl2008analysis}. 
\citet{abbasi2014online} studied a problem similar to our linear combination setting, and proposed a no-regret algorithm by combining UCRL2 \citep{jaksch2010near} with confidence set techniques from stochastic linear optimization literature \citep{dani2008stochastic,filippi2010parametric}. Our work takes an independent and very different approach, and we provide a PAC guarantee which is not directly comparable to regret bound. Still, we observe that our dependence on $A$ is optimal for PAC whereas theirs is not ($\sqrt{A}$ is optimal for bandit regret analysis and they have $A$); on the other hand, their dependence on $T$ (the number of rounds) is optimal, and our dependence on $1/\epsilon$, its counterpart in PAC analysis, is suboptimal. It is an interesting future direction to combine the algorithmic ideas from both papers to improve the guarantees.

\section{Conclusion}
In this paper, we present a general setting of using side information for learning near-optimal policies in a large and potentially infinite number of MDPs. The proposed Cover-Rmax algorithm is a model-based PAC-exploration algorithm for the case where MDPs vary smoothly with respect to the observed side information. Our lower bound construction indicates the necessary exponential dependence of any PAC algorithm on the context dimension in a smooth CMDP. We also consider another instance with a parametric assumption, and using a KWIK linear regression procedure, present the KWIK\_LR-Rmax algorithm for efficient exploration in linear combination of MDPs. Our PAC analysis shows a significant improvement with this structural assumption.

The use of context based modelling of multiple tasks has rich application possibilities in personalized recommendations, healthcare treatment policies, and tutoring systems. We believe that our setting can possibly be extended to cover the large space of multi-task RL quite well with finite/infinite number of MDPs, observed/latent contexts, and deterministic/noisy mapping between context and environment. We hope our work spurs further research along these directions.

\begin{acks}
%%%%%%%%%%%%%
%           %
% IMPORTANT %
%           %
%%%%%%%%%%%%%
% if papers is accepted, make sure the right grants are ack'ed.
This work was supported in part by a grant from the Open Philanthropy Project to the Center for Human-Compatible AI, and in part by NSF Grant IIS 1319365. Ambuj Tewari acknowledges the support from NSF grant CAREER IIS-1452099 and Sloan Research Fellowship. Any opinions, findings, conclusions, or recommendations expressed here are those of the authors and do not necessarily reflect the views of the sponsors. 
\end{acks}

% \newpage
\bibliographystyle{alpha}
\bibliography{ALT_2018}

\begin{thebibliography}{26}
\providecommand{\natexlab}[1]{#1}
\providecommand{\url}[1]{\texttt{#1}}
\expandafter\ifx\csname urlstyle\endcsname\relax
  \providecommand{\doi}[1]{doi: #1}\else
  \providecommand{\doi}{doi: \begingroup \urlstyle{rm}\Url}\fi

\bibitem[Abbasi-Yadkori and Neu(2014)]{abbasi2014online}
Yasin Abbasi-Yadkori and Gergely Neu.
\newblock Online learning in mdps with side information.
\newblock \emph{arXiv preprint arXiv:1406.6812}, 2014.

\bibitem[Ammar et~al.(2014)Ammar, Eaton, Ruvolo, and Taylor]{ammar2014online}
Haitham~B Ammar, Eric Eaton, Paul Ruvolo, and Matthew Taylor.
\newblock Online multi-task learning for policy gradient methods.
\newblock In \emph{Proceedings of the 31st International Conference on Machine
  Learning (ICML-14)}, pages 1206--1214, 2014.

\bibitem[Brafman and Tennenholtz(2002)]{brafman2002r}
Ronen~I Brafman and Moshe Tennenholtz.
\newblock R-max-a general polynomial time algorithm for near-optimal
  reinforcement learning.
\newblock \emph{Journal of Machine Learning Research}, 3\penalty0
  (Oct):\penalty0 213--231, 2002.

\bibitem[Brunskill and Li(2013)]{brunskill2013sample}
Emma Brunskill and Lihong Li.
\newblock Sample complexity of multi-task reinforcement learning.
\newblock \emph{arXiv preprint arXiv:1309.6821}, 2013.

\bibitem[Dani et~al.(2008)Dani, Hayes, and Kakade]{dani2008stochastic}
Varsha Dani, Thomas~P Hayes, and Sham~M Kakade.
\newblock Stochastic linear optimization under bandit feedback.
\newblock In \emph{COLT}, pages 355--366, 2008.

\bibitem[Dann and Brunskill(2015)]{dann2015sample}
Christoph Dann and Emma Brunskill.
\newblock Sample complexity of episodic fixed-horizon reinforcement learning.
\newblock In \emph{Advances in Neural Information Processing Systems}, pages
  2818--2826, 2015.

\bibitem[Duchi et~al.(2008)Duchi, Shalev-Shwartz, Singer, and
  Chandra]{duchi2008efficient}
John Duchi, Shai Shalev-Shwartz, Yoram Singer, and Tushar Chandra.
\newblock Efficient projections onto the l 1-ball for learning in high
  dimensions.
\newblock In \emph{Proceedings of the 25th international conference on Machine
  learning}, pages 272--279. ACM, 2008.

\bibitem[Filippi et~al.(2010)Filippi, Cappe, Garivier, and
  Szepesv{\'a}ri]{filippi2010parametric}
Sarah Filippi, Olivier Cappe, Aur{\'e}lien Garivier, and Csaba Szepesv{\'a}ri.
\newblock Parametric bandits: The generalized linear case.
\newblock In \emph{Advances in Neural Information Processing Systems}, pages
  586--594, 2010.

\bibitem[Hallak et~al.(2015)Hallak, Di~Castro, and
  Mannor]{hallak2015contextual}
Assaf Hallak, Dotan Di~Castro, and Shie Mannor.
\newblock Contextual markov decision processes.
\newblock \emph{arXiv preprint arXiv:1502.02259}, 2015.

\bibitem[Jaksch et~al.(2010)Jaksch, Ortner, and Auer]{jaksch2010near}
Thomas Jaksch, Ronald Ortner, and Peter Auer.
\newblock Near-optimal regret bounds for reinforcement learning.
\newblock \emph{Journal of Machine Learning Research}, 11\penalty0
  (Apr):\penalty0 1563--1600, 2010.

\bibitem[Kakade(2003)]{kakade2003sample}
Sham~Machandranath Kakade.
\newblock \emph{On the sample complexity of reinforcement learning}.
\newblock PhD thesis, 2003.

\bibitem[Kearns and Singh(2002)]{kearns2002near}
Michael Kearns and Satinder Singh.
\newblock Near-optimal reinforcement learning in polynomial time.
\newblock \emph{Machine Learning}, 49\penalty0 (2-3):\penalty0 209--232, 2002.

\bibitem[Killian et~al.(2016)Killian, Konidaris, and
  Doshi-Velez]{killian2016transfer}
Taylor Killian, George Konidaris, and Finale Doshi-Velez.
\newblock Transfer learning across patient variations with hidden parameter
  markov decision processes.
\newblock \emph{arXiv preprint arXiv:1612.00475}, 2016.

\bibitem[Langford and Zhang(2008)]{langford2008epoch}
John Langford and Tong Zhang.
\newblock The epoch-greedy algorithm for multi-armed bandits with side
  information.
\newblock In \emph{Advances in neural information processing systems}, pages
  817--824, 2008.

\bibitem[Lazaric(2011)]{lazaric2011transfer}
A.~Lazaric.
\newblock Transfer in reinforcement learning: a framework and a survey.
\newblock In M.~Wiering and M.~van Otterlo, editors, \emph{Reinforcement
  Learning: State of the Art}. Springer, 2011.

\bibitem[Li et~al.(2008)Li, Littman, and Walsh]{li2008knows}
Lihong Li, Michael~L Littman, and Thomas~J Walsh.
\newblock Knows what it knows: a framework for self-aware learning.
\newblock In \emph{Proceedings of the 25th international conference on Machine
  learning}, pages 568--575. ACM, 2008.

\bibitem[Li et~al.(2010)Li, Chu, Langford, and Schapire]{li2010contextual}
Lihong Li, Wei Chu, John Langford, and Robert~E Schapire.
\newblock A contextual-bandit approach to personalized news article
  recommendation.
\newblock In \emph{Proceedings of the 19th international conference on World
  wide web}, pages 661--670. ACM, 2010.

\bibitem[Mahmud et~al.(2013)Mahmud, Hawasly, Rosman, and
  Ramamoorthy]{mahmud2013clustering}
MM~Mahmud, Majd Hawasly, Benjamin Rosman, and Subramanian Ramamoorthy.
\newblock Clustering markov decision processes for continual transfer.
\newblock \emph{arXiv preprint arXiv:1311.3959}, 2013.

\bibitem[Pazis and Parr(2013)]{pazis2013pac}
Jason Pazis and Ronald Parr.
\newblock Pac optimal exploration in continuous space markov decision
  processes.
\newblock In \emph{AAAI}, 2013.

\bibitem[Slivkins(2014)]{slivkins2014contextual}
Aleksandrs Slivkins.
\newblock Contextual bandits with similarity information.
\newblock \emph{Journal of Machine Learning Research}, 15\penalty0
  (1):\penalty0 2533--2568, 2014.

\bibitem[Strehl and Littman(2008)]{strehl2008analysis}
Alexander~L Strehl and Michael~L Littman.
\newblock An analysis of model-based interval estimation for markov decision
  processes.
\newblock \emph{Journal of Computer and System Sciences}, 74\penalty0
  (8):\penalty0 1309--1331, 2008.

\bibitem[Strehl et~al.(2009)Strehl, Li, and Littman]{strehl2009reinforcement}
Alexander~L Strehl, Lihong Li, and Michael~L Littman.
\newblock Reinforcement learning in finite mdps: Pac analysis.
\newblock \emph{Journal of Machine Learning Research}, 10\penalty0
  (Nov):\penalty0 2413--2444, 2009.

\bibitem[Szita and Szepesv{\'a}ri(2011)]{szita2011agnostic}
Istv{\'a}n Szita and Csaba Szepesv{\'a}ri.
\newblock Agnostic kwik learning and efficient approximate reinforcement
  learning.
\newblock In \emph{Proceedings of the 24th Annual Conference on Learning
  Theory}, pages 739--772, 2011.

\bibitem[Taylor and Stone(2009)]{taylor2009transfer}
Matthew~E Taylor and Peter Stone.
\newblock Transfer learning for reinforcement learning domains: A survey.
\newblock \emph{Journal of Machine Learning Research}, 10\penalty0
  (Jul):\penalty0 1633--1685, 2009.

\bibitem[Walsh et~al.(2009)Walsh, Szita, Diuk, and Littman]{walsh2009exploring}
Thomas~J Walsh, Istv{\'a}n Szita, Carlos Diuk, and Michael~L Littman.
\newblock Exploring compact reinforcement-learning representations with linear
  regression.
\newblock In \emph{Proceedings of the Twenty-Fifth Conference on Uncertainty in
  Artificial Intelligence}, pages 591--598. AUAI Press, 2009.
\newblock A corrected version is available as Technical Report DCS-tr-660,
  Department of Computer Science, Rutgers University, December, 2009.

\bibitem[Wu(2016)]{cov_num}
Yihong Wu.
\newblock Packing, covering, and consequences on minimax risk.
\newblock
  \href{http://www.stat.yale.edu/~yw562/teaching/598/lec14.pdf}{EECS598:
  Information-theoretic methods in high-dimensional statistics}, Statistics,
  Yale University, 2016.

\end{thebibliography}

%\ambuj{Capitalize references to numbered sections/theorems/lemmas}

\newpage
\appendix
\section{Proofs from Section~\ref{sec:covrmax}}
\subsection{Proof of Lemma~\ref{lem:com}}
We adapt the analysis in \cite{kakade2003sample} for the episodic case which results in the removal of a factor of $H$, since complete episodes are counted as mistakes and we do not count every sub-optimal action in an episode. We reproduce the detailed analysis here for completion.
For completing the proof of Lemma~\ref{lem:com}, firstly, we will look at a version of simulation lemma from \cite{kearns2002near}. Also, for the complete analysis we will assume that the rewards lie between 0 and 1.

\begin{definition}[Induced MDP] 
\label{def:ind}
Let $M$ be an MDP with $K \subseteq \C{S}$ being a subset of states. Given, such a set $K$, we define an induced MDP $M_K$ in the following manner. For each $s \in K$, define the values
\begin{align*}
    p_{M_K}(s'|s,a) = p_M(s'|s,a)\\
    r_{M_K} (s,a) = r_M(s,a)
\end{align*}
For all $s \notin K$, define $p_{M_K}(s'|s,a) = \B{I}\{s' = s\}$ and $r_{M_K} (s,a) = 1$.
\end{definition}

\begin{lemma}[Simulation lemma for episodic MDPs]
\label{lem:sim}
Let $M$ and $M'$ be two MDPs with the same state-action space. If the transition dynamics and the reward functions of the two MDPs are such that 
\begin{align*}
\|p_M(\cdot |s,a) - p_{M'}(\cdot |s,a)\|_1 \leq \epsilon_1 \, \, \forall s \in \C{S}, a \in \C{A}
\end{align*}
\begin{align*}
|r_M(s,a) - r_{M'}(s,a)| \leq \epsilon_2 \, \, \forall s \in \C{S}, a \in \C{A}
\end{align*}
then, for every (non-stationary) policy $\pi$ the two MDPs satisfy this property:
\begin{align*}
    |V^{\pi}_M - V^{\pi}_{M'}| \leq \epsilon_2 + H\epsilon_1
\end{align*}
\end{lemma}
\begin{proof}
Consider $\C{T}_h$ to be the set of all trajectories of length $h$ and let $P^{\pi}_M(\tau)$ denote the probability of observing trajectory $\tau$ in $M$ with the behaviour policy $\pi$. Further, let $U_M(\tau)$ the expected average reward obtained for trajectory $\tau$ in MDP $M$.
\begin{eqnarray*}
|V^{\pi}_M - V^{\pi}_{M'}| &=& |\sum_{\tau \in \C{T}_H}\, [P^{\pi}_M (\tau) U_M (\tau) - P^{\pi}_{M'} (\tau) U_{M'} (\tau)]\,| \\
&\leq& |\sum_{\tau \in \C{T}_H} \,[P^{\pi}_M (\tau) U_M (\tau) - P^{\pi}_M (\tau) U_{M'} (\tau) + P^{\pi}_M (\tau) U_{M'} (\tau) - P^{\pi}_{M'} (\tau) U_{M'} (\tau)]\,|\\
&\leq& |\sum_{\tau \in \C{T}_H}\, [P^{\pi}_M (\tau) (U_{M} (\tau) - U_{M'} (\tau))]\,| + |\sum_{\tau \in \C{T}_H} \,[U_{M'} (\tau) ( P^{\pi}_{M} (\tau) - P^{\pi}_{M'} (\tau))] \,| \\ 
&\leq& |\sum_{\tau \in \C{T}_H} P^{\pi}_M (\tau)| \epsilon_2 + |\sum_{\tau \in \C{T}_H}\, [ P^{\pi}_{M} (\tau) - P^{\pi}_{M'} (\tau)] \,|\\
&\leq& \epsilon_2 + |\sum_{\tau \in \C{T}_H}\, [ P^{\pi}_{M} (\tau) - P^{\pi}_{M'} (\tau)] \,|
\end{eqnarray*}
The bound for the second term follows from the proof of lemma 8.5.4 in \cite{kakade2003sample}. Combining the two expressions, we get the desired result.
\end{proof}

\begin{lemma}[Induced inequalities]
\label{lem:indineq}
Let $M$ be an MDP with $K$ being the set of known states. Let $M_K$ be the induced MDP as defined in \ref{def:ind} with respect to $K$ and $M$. We will show that for any (non-stationary) policy $\pi$, all states $s \in \C{S}$,
\begin{align*}
    V^{\pi}_{M_K}(s) \geq V^{\pi}_{M}(s)
\end{align*}
and
\begin{align*}
    V^{\pi}_{M} (s) \geq V^{\pi}_{M_K}(s) - \text{P}^{\pi}_M[\text{Escape to an unknown state}|s_0 = s]
\end{align*}
where $V^{\pi}_{M} (s)$ denotes the value of policy $\pi$ in MDP $M$ when starting from state $s$.
\end{lemma}
\begin{proof}
See Lemma 8.4.4 from \cite{kakade2003sample}.
\end{proof}

\begin{corollary}[Implicit Explore and Exploit]
\label{cor:actual}
Let $M$ be an MDP with $K$ as the set of known states and $M_K$ be the induced MDP. If $\pi^*_{M_K}$ and $\pi^*_M$ be the optimal policies for $M_K$ and $M$ respectively, we have for all states $s$:
\begin{align*}
    V^{\pi^*_{M_K}}_M(s) \geq V^{*}_M(s) - \text{P}^{\pi}_M[\text{Escape to an unknown state}|s_0 = s]
\end{align*}
\end{corollary}
\begin{proof}
Follows from Lemma 8.4.5 from \cite{kakade2003sample}.
\end{proof}

\begin{proof}[\textbf{Proof of Lemma~\ref{lem:com}}]
Let $\pi^*_{M}$ be the optimal policy for $M$. Also, using the assumption about $m$, we have an $\epsilon/2$-approximation of $M_K$ as the MDP $\hat{M}_K$. Rmax computes the optimal policy for $\hat{M_K}$ which is denoted by $\hat{\pi}$. Then, by Lemma~\ref{lem:sim},
\begin{eqnarray*}
    V^{\hat{\pi}}_{M_K}(s) &\geq& V^{\hat{\pi}}_{\hat{M}_K}(s) - \epsilon/2 \\
    &\geq& V^{\pi^*_{M}}_{\hat{M}_K}(s) - \epsilon/2\\
    &\geq& V^{\pi^*_{M}}_{M_K}(s) - \epsilon
\end{eqnarray*}
Combining this with Lemma~\ref{lem:indineq}, we get
\begin{eqnarray*}
    V^{\hat{\pi}}_{M}(s) &\geq& V^{\hat{\pi}}_{M_K}(s) - \text{P}^{\pi}_M[\text{Escape to an unknown state}|s_0 = s]\\
    &\geq& V^{\pi^*_{M}}_{M_K}(s) - \epsilon - \text{P}^{\pi}_M[\text{Escape to an unknown state}|s_0 = s]\\
    &\geq& V^*_{M}(s) - \epsilon - \text{P}^{\pi}_M[\text{Escape to an unknown state}|s_0 = s]
\end{eqnarray*}
If this escape probability is less than $\epsilon$, then the desired relation is true.
Therefore, we need to bound the number of episodes where this expected number is greater than $\epsilon$. Note that, due to balanced wandering, we can have at most $mSA$ visits to unknown states for the Rmax algorithm. In the execution, we may encounter an extra $H-1$ visits as the estimates are updated only after the termination of an episode.

Whenever this quantity is more than $\epsilon$, the expected number of exploration steps in $mSA/\epsilon$ such episodes is at least $mSA$. By the Hoeffding's inequality, for $N$ episodes, with probability, at least $1-\delta$, the number of successful exploration steps is greater than 
\begin{align*}
N\epsilon - \sqrt{\frac{N}{2} \ln{\frac{1}{\delta}}}
\end{align*}
Therefore, if $N = \C{O}(\frac{mSA}{\epsilon}\ln{\frac{1}{\delta}})$, with probability at least $1-\delta$, the total number of visits to an unknown state is more than $mSA$. Using the upper bound on such visits, we conclude that these many episodes suffice.
\end{proof}

\subsection{Proof of Theorem~\ref{thm:pac}}
We now need to compute the required resolution of the cover and the number of transitions $m$ which will guarantee the approximation for the value functions as required in the previous lemma. The following result is the key result:

\begin{lemma}[Cover approximation]
\label{lem:covapp}
For a given CMDP and a finite cover, i.e., $\C{C} = \cup_{i=1}^{N(\C{C},r)} \C{B}_i$ such that $\forall i,\, \forall c_1,c_2 \in \C{B}_i$ :
\begin{align*}
    \|p^{c_1}(\cdot|s,a) - p^{c_2}(\cdot|s,a)\|_1 \leq \epsilon/8H
\end{align*}
and
\begin{align*}
    |r^{c_1}(s,a) - r^{c_2}(s,a)| \leq \epsilon/8
\end{align*}
if we visit every state-action pair $m = \frac{128(S\ln 2+\ln{\frac{SA}{\delta})H^2}}{\epsilon^2}$ times in a ball $\C{B}_i$ summing observations over all $c \in \C{B}_i$, then, for any policy $\pi$ and with probability at least $1-2\delta$, the approximate MDP $\hat{M}_i$ corresponding to $\C{B}_i$ computed using empirical averages will satisfy
\begin{align*}
    |V^\pi_{M_c} - V^\pi_{\hat{M}_i}| \leq \epsilon/2
\end{align*}
for all $c \in \C{B}_i$.
\end{lemma}
\begin{proof}

For each visit to a state action pair $(s,a)$, we observe a transition to some $s' \in \C{S}$ for context $c_t \in \C{B}_i$ in $t_{\text{th}}$ visit with probability $p_{c_t}(s,a)$. Let us encode this by an $S$-dimensional vector $I_t$ with $0$ at all indices except $s'$. After observing $m$ such transitions, we create the estimate for any $c \in \C{B}_i$ as $p_{\hat{M}_i}(\cdot|s,a) = \tfrac{1}{m} \sum_{t=1}^{m} I_t$.
Now for all $c \in \C{B}_i$, 
\begin{align*}
    \|p_{\hat{M}_i}(\cdot|s,a) - p^c(\cdot|s,a)\|_1 \leq \|p_{\hat{M}_i}(\cdot|s,a) - \frac{1}{m}\sum_{t=1}^{m} p^{c_t}(\cdot|s,a)\|_1 + \epsilon/8H
\end{align*}
For bounding the first term, we use the Hoeffding's bound:
\begin{small}
\begin{eqnarray*}
P\Big[\,\|p_{\hat{M}_i}(\cdot|s,a) - \frac{1}{m}\sum_{t=1}^{m} p^{c_t}(\cdot|s,a)\|_1 \geq \epsilon\Big] &=& P\Big[\max_{s' \in A \subseteq \C{S}} (p_{\hat{M}_i}(s' \in A|s,a) - \frac{1}{m}\sum_{t=1}^{m} p^{c_t}(s' \in A|s,a)) \geq \epsilon/2\Big] \\
&\leq& \sum_{s' \in A \subseteq \C{S}} P\Big[ (p_{\hat{M}_i}(s' \in A|s,a) - \frac{1}{m}\sum_{t=1}^{m} p^{c_t}(s' \in A|s,a)) \geq \epsilon/2\Big]\\
&\leq& (2^S - 2) \exp (-m\epsilon^2/2)
\end{eqnarray*}
\end{small}
Therefore, with probability at least $1-\delta/2$, for all $s \in \C{S}, a \in \C{A}$, we have:
\begin{align*}
    \|p_{\hat{M}_i}(\cdot|s,a) - p^c(\cdot|s,a)\|_1 \leq \sqrt{\frac{2(S \ln 2 + \ln{2SA/\delta})}{m}} + \epsilon/8H
\end{align*}

If $m = \frac{128(S\ln 2+\ln{\frac{SA}{\delta})H^2}}{\epsilon^2} $, the error becomes $\epsilon/4H$. One can easily verify using similar arguments that, the error in rewards for any context $c \in \C{B}_i$ is less than $\epsilon/4$. 

By using the simulation lemma~\ref{lem:sim}, we get the desired result.
\end{proof}

\section{Lower bound analysis}
\subsection{Proof of Claim~\ref{cl:val_SCMDP}}
Once the instances at the packing points are assigned, the parameters for any other context $c \in \C{C}$, state $i$ and action $a$ are given by:
\begin{align*}
p_c(+|i,a) = \max_{c' \in Z} \max (1/2, p_{c'}(+|i,a) - \dis{c}{c'}/2)
\end{align*}
We now prove that, with this definition, the smoothness requirements are satisfied:
\begin{claim}
The contextual MDP defined above is a valid instance of a contextual MDP with smoothness constants $L_p = 1$. (The reward function does not vary with context hence reward smoothness is satisfied for all $L_r\ge 0$.).
\end{claim}
\begin{proof}
We need to prove that the defined contextual MDP, satisfies the constraints in Definition \ref{def:smcmdp}. Let us assume that the smoothness assumption is violated for a context pair $(c_1,c_2)$. The smoothness constraints for rewards are satisfied trivially for any value of $L_r$ as they are constant. This implies that there exists state $i \in [n]$ and action $a$ such that
\begin{align*}
\|p_{c_1}(\cdot|i,a) - p_{c_2}(\cdot|i,a)\|_1 > \dis{c_1}{ c_2}\\
\Rightarrow 2|p_{c_1}(+|i,a) - p_{c_2}(+|i,a)| > \dis{c_1}{ c_2}
\end{align*}
We know that, $p_c(+|i,a) \in [1/2, 1/2+\epsilon']$, which shows that $\dis{c_1}{ c_2} < 2\epsilon'$. Without loss of generality, assume $p_{c_1}(+|i,a) > p_{c_2}(+|i,a)$ which also leads to 
\begin{align*}
p_{c_1}(+|i,a) > 1/2\\
\Rightarrow \exists c_0 \in Z \text{ such that } \dis{c_1}{ c_0} < 2\epsilon'
\end{align*}
By triangle inequality, we have
\begin{align*}
    \dis{c_2}{ c_0} < 4\epsilon'
\end{align*}
Now, $\forall c_0' \in Z$, such that $\dis{c_0'}{c_0} \geq 8\epsilon'$, by triangle inequality, we have
\begin{align*}
    \dis{c_0'}{c_1} > 6\epsilon'\\
    \dis{c_0'}{c_2} > 4\epsilon'
\end{align*}
This simplifies the definition of $p_c(+|i,a)$ for $c=c_1,c_2$ to
\begin{align*}
    p_c(+|i,a) = \max(1/2, p_{c_0}(+|i,a) - \dis{c_0}{c}/2)
\end{align*}
Now,
\begin{eqnarray*}
|p_{c_1}(+|i,a) - p_{c_2}(+|i,a)| &=& p_{c_1}(+|i,a) - p_{c_2}(+|i,a)\\
&=& p_{c_0}(+|i,a) - \dis{c_0}{c_1}/2 - \max(1/2, p_{c_0}(+|i,a) - \dis{c_0}{c_2}/2)\\
&\leq& p_{c_0}(+|i,a) - \dis{c_0}{c_1}/2 - (p_{c_0}(+|i,a) - \dis{c_0}{c_2}/2)\\
&=& \tfrac{1}{2}(\dis{c_0}{c_2} - \dis{c_0}{c_1}) \leq \dis{c_1}{c_2}/2
\end{eqnarray*}
which leads to a contradiction.
\end{proof}

\subsection{Lower bound for smooth CMDP}
In the lower bound construction, we made a key argument that the defined contextual mapping with specific instances, requires any agent to learn the models of a significant number of MDPs separately to get decent generalization. The construction populates a set of packing points in the context space with hard MDPs and argues that these instances are independent of each other from the algorithm's perspective. To formalize this statement, let $Z$ be the $8\epsilon'$-packing as before. The adversary makes the choices of the instances at each context $c \in Z$, as follows: Select an MDP from the family of hard instances described in Figure~\ref{fig:lowbnd} where the optimal action from each state in $\{1,\ldots, n\}$ is chosen randomly and independently from the other assignments. The parameter $\epsilon'$ deciding the difference in optimality of actions in Figure~\ref{fig:lowbnd} is taken as $\tfrac{160 H \epsilon e^4}{(H-2)}$ which is required by the construction of Theorem~\ref{thm:mdp_low}.

We denote these instances by the set $\C{I}$ and an individual instance by $\C{I}_z$ where $z = \{z_1, z_2,\ldots,z_{|Z|}\}$ is the set of these random choices made by the adversary. By construction, we have a uniform distribution $\Gamma \equiv \Gamma_1 \times \Gamma_2 \times \ldots \Gamma_{|Z|}$ over these possible set of instances $\C{I}_z$. From claim 5.3, any assignment of optimal actions to these packing points would define a valid smooth contextual MDP. Further, the independent choice of the optimal actions makes MDPs at each packing point at least as difficult as learning a single MDP. Formally, let the sequence of transitions and rewards observed by the learning agent  for all packing points be $T \equiv \{\tau_1, \tau_2, \ldots, \tau_{|Z|}\}$. Due to the independence between individual instances, we can see that:
\begin{align*}
    P_{\Gamma}[\,\tau_1|\tau_2, \tau_3, \ldots, \tau_{|Z|}\,] \equiv P_{\Gamma_1}[\,\tau_1\,]
\end{align*}
where $P_{\Gamma}[\tau_i]$ denotes the distribution of trajectories $\tau_i$.
Thus, observing trajectories from MDP instances at other packing points does not let the algorithm deduce anything about the nature of the instance chosen for one point. With respect to this distribution, learning in the contextual MDP is equivalent to or worse than simulating a single MDP learning algorithm at each of these packing points. For any given contextual MDP algorithm $\texttt{Alg}$, we have:
\begin{align*}
    \B{E}_{\Gamma}[B_i] = \B{E}_{\Gamma_{-i}}[\B{E}_{\Gamma_i, \texttt{Alg}}[B_i | T_{-i}]] \geq \B{E}_{\Gamma_{-i}}[\B{E}_{\Gamma_1, \texttt{Alg*}} [B_i]]
\end{align*}
where \texttt{Alg*} is an optimal single MDP learning algorithm. The expectation is with respect to the distribution over the instances $\C{I}_{z}$ and the algorithm's randomness. From Theorem~\ref{thm:mdp_low}, we can lower bound the expectation on the right hand side of the inequality by $\Omega\Big( \tfrac{SA}{\epsilon^2}\Big)$.

The total number of mistakes is lower bounded as:
\begin{align*}
\B{E}_{\Gamma} [\sum_{i=1}^{|Z|} B_i] = \sum_{i=1}^{|Z|} \B{E}_{\Gamma}[B_i] \geq \Omega\Big( \tfrac{|Z|SA}{\epsilon^2}\Big)
\end{align*}

Setting $|Z| = \C{D}(\C{C},\epsilon_1) \leq \C{N}(\C{C},\epsilon_1)$ gives the stated lower bound with $\epsilon_1 = 8\epsilon'$.

\section{Proof of the Theorem~\ref{thm:kwiklr}}
In this section, we present the proof of our KWIK bound for learning transition probabilities. Our proof uses a reduction technique that reduces the vector-valued label setting to the scalar setting, and combines the KWIK bound for scalar labels given by \citet{walsh2009exploring}.  %the KWIK bound for learning scalar linear regression outputs 
\begin{proof}
Fix a state action pair $(s,a)$. Consider a sequence of contexts $c_1, c_2,...$ for which the transitions were observed for pair $(s,a)$. Given a new context $c$, we want to estimate:
\begin{align*}
    p^c(\cdot|s,a) = c^\top P(s,a)
\end{align*}

In our KWIK\_LR algorithm, we estimate this as:
\begin{align*}
    \hat{p}^c(\cdot|s,a) = c^\top Q(s,a)W(s,a)
\end{align*}
where $Q(s,a)$ and $W(s,a)$ are as described in Section~\ref{sec:kwikrmax}.

We wish to bound the $\ell_1$ error between $\hat{p}^c(\cdot|s,a)$ and $p^c(\cdot|s,a)$ for all $c$ for which a prediction is made. We know that
\begin{align} \label{eq:l1_f}
    \|p^c(\cdot|s,a) - \hat{p}^c(\cdot|s,a)\|_1 = \sup_{f \in \{-1,1\}^S} \,(p^c(\cdot|s,a) - \hat{p}^c(\cdot|s,a))f.
\end{align}
We use this representation of $\ell_1$-norm to prove a tighter KWIK bound for learning transition probabilities. For every fixed $f \in \{-1,1\}^S$, we formulate a new linear regression problem with feature-label pair:
\begin{align*}
    (c, y f).
\end{align*}
Recall that $y = (\{\B{I}[s_{\text{next}} = s']\}_{\forall s' \in \C{S}})^\top$ is the vector label of real interest, and $f$ projects $y$ to a scalar value. 
Algorithm~\ref{alg:KWIK} can be viewed as implicitly running this regression thanks to linearity: since $Q$ only depends on input contexts and $W$ is linear in $y$, $\hat{p}^c(\cdot|s,a) f$ is simply equal to the linear regression prediction for the problem $(c, yf)$. As a result, the KWIK bound for the problem $(c, yf)$ (which we establish below) automatically applies as a property of $\hat{p}^c(\cdot|s,a) f$. Taking union bound over all $f \in \{-1, 1\}^S$ yields the desired $\ell_1$ error guarantee for $\hat{p}^c(\cdot|s,a)$ thanks to Equation~\ref{eq:l1_f}.

Now we establish the KWIK guarantee for the new regression problem. The groundtruth (expected) label is
\begin{align}
\label{eq:theta^f}
p^c(\cdot|s,a) f = c^\top (P(s,a) f) := c^\top \theta^f.
\end{align}
The noise in the label is then
\begin{align}
\label{eq:nlr}
    \eta^f := (y - p^c(\cdot|s,a)) f.
\end{align}
This noise has zero-mean and constant magnitude: $|\eta^f|  \le \|y - p^c(\cdot|s,a)\|_1 \|f\|_{\infty} \leq 2$.

With the above conditions, we can invoke the KWIK bound for scalar linear regression from \cite{walsh2009exploring}:
\begin{theorem}[KWIK bound for linear regression\citep{walsh2009exploring}]
\label{thm:walsh}
Suppose the observation noise in a noisy linear regression problem has zero-mean and its absolute value is bounded by $\beta$. Let $M$ be an upper bound on the $\ell_2$ norm of the true linear coefficients. For any $\delta' >0$ and $\epsilon > 0$, if the KWIK linear regression algorithm is executed with $\alpha_0 = \min \{ b_1 \tfrac{\epsilon^2}{dM}, b_2 \tfrac{\epsilon^2}{M\log (d/\delta')} , \tfrac{\epsilon}{2M}\}$, with suitable constants $b_1$ and $b_2$, then the number of $\perp$'s will be $O( M^2 \max \{ \tfrac{d^3}{\epsilon^4}, \tfrac{d \log^2(d/\delta')}{\epsilon^4}\} )$, and with probability at least $1-\delta'$, for each sample $x_t$ for which a prediction is made, the prediction is $\epsilon$-accurate.
\end{theorem}
For our purpose, $\beta = 2$ as $|\eta^f|\le 2$ and $M = \sqrt{d}$ as $\|\theta^f\|_2 = \|P(s,a)f\|_2 \leq \sqrt{d}$.

Now set $\delta' = \tfrac{\delta}{2^S}$ in Theorem~\ref{thm:walsh}. In the KWIK linear regression algorithm, the \emph{known} status for a context $c$ is checked in the same manner as done in Line~\ref{lin:kwik_known} in Algorithm~\ref{alg:KWIK}. Therefore
\begin{align*}
&~ Pr\Big[\|p^c(\cdot|s,a) - \hat{p}^c(\cdot|s,a)\|_1 \ge \epsilon\Big] \\
=&~ Pr\Big[\sup_{f \in \{-1,1\}^S}  \,(p^c(\cdot|s,a) - \hat{p}^c(\cdot|s,a))f \ge \epsilon\Big] \tag{Equation~\ref{eq:l1_f}}\\
\le&~ \sum_{f \in \{-1,1\}^S} Pr\Big[(p^c(\cdot|s,a) - \hat{p}^c(\cdot|s,a))f \ge \epsilon\Big] \tag{union bound}\\
=&~ \sum_{f \in \{-1,1\}^S} Pr\Big[c^\top \theta^f - c^\top Q(s,a) (W(s,a)f) \ge \epsilon\Big] \tag{regression w.r.t.~$f$ implicitly run}\\
\le&~ \sum_{f \in \{-1,1\}^S} \delta / 2^S
= \delta.
\end{align*}

Substituting the values of $M$ and $\delta'$ in Theorem~\ref{thm:walsh}, we get:
\begin{align*}
    \alpha_S = \min \{ b_1 \tfrac{\epsilon^2}{d^{3/2}}, b_2 \tfrac{\epsilon^2}{\sqrt{d}\log (d/2^S \delta')} , \tfrac{\epsilon}{2\sqrt{d}}\}
\end{align*}
and number of $\perp$'s is bounded as
\begin{align*}
    O(  \max \{ \tfrac{d^4}{\epsilon^4}, \tfrac{d^2S^2 \log^2(d/\delta')}{\epsilon^4}\} ). \tag*{\qedhere}
\end{align*}
\end{proof}

\end{document}